\documentclass{article}

\usepackage{arxiv}

\usepackage[utf8]{inputenc} 
\usepackage[T1]{fontenc}    
\usepackage[colorlinks=true, linkcolor=blue, urlcolor=blue, citecolor=blue]{hyperref}     
\usepackage{url}            
\usepackage{booktabs}       
\usepackage{amsfonts}       
\usepackage{nicefrac}       
\usepackage{microtype}      
\usepackage{lipsum}		
\usepackage{natbib}
\usepackage{doi}

\usepackage{graphicx}

\usepackage{xcolor} 

\usepackage{amsthm}

\usepackage{amsmath}
\usepackage{amssymb}
\usepackage[english]{babel}
\usepackage{amsthm}
\usepackage{graphicx}
\usepackage{caption}
\usepackage{subcaption}

\usepackage[bottom]{footmisc}
\usepackage{bbm}
\usepackage{algorithm}
\usepackage{algpseudocode}

\newtheorem{theorem}{Theorem}

\newtheorem{corollary}[theorem]{Corollary}

\theoremstyle{definition}
\newtheorem{definition}[theorem]{Definition}
\newtheorem{remark}[theorem]{Remark}
\numberwithin{figure}{section}

\title{A simple algorithm for output range analysis for deep neural networks}

\author{{Helder Rojas} \\
	Department of Mathematics\\
	Imperial College London\\
	London, United Kingdom \\
	\href{mailto:me@somewhere.com}{h.rojas-molina23@imperial.ac.uk} \\
\And
	{Nilton Rojas} \\
	Escuela Profesional de Ciencias de la Computación\\
	Universidad Nacional de Ingeniería\\
	Lima, Perú \\
	\href{mailto:me@somewhere.com}{nrojasv@uni.pe} \\
  \And
	{Espinoza J. B.} \\
	Escuela Profesional de Ingeniería Estadística\\
	Universidad Nacional de Ingeniería\\
	Lima, Perú\\
	\href{mailto:me@somewhere.com}{jespinozas@uni.edu.pe} \\
 \And
        {Luis Huamanchumo} \\
	Escuela Profesional de Ingeniería Estadística\\
	Universidad Nacional de Ingeniería\\
	Lima, Perú \\
	\href{mailto:me@somewhere.com}{lhuamanchumo@uni.edu.pe} \\
}

\renewcommand{\shorttitle}{OUTPUT RANGE ANALYSIS FOR DEEP NEURAL NETWORK}

\hypersetup{
pdftitle={Output Range Analysis for Deep Neural Networks based on Simulated Annealing Processes},
pdfauthor={Helder Rojas, Nilton Rojas, Espinoza J. B., Luis Huamanchumo},
pdfkeywords={Deep Neural Networks, Simulated Annealing, Output Range Estimation, Residual Neural Networks, Global Optimization, Cyber-Physical Systems, Non-Convex Optimization},
}

\begin{document}
\maketitle

\begin{abstract}
This paper presents a simple algorithm for the output range estimation problem in Deep Neural Networks (DNNs) by integrating a Simulated Annealing (SA) algorithm tailored to operate within constrained domains and ensure convergence towards global optima. The method effectively addresses the challenges posed by the lack of local geometric information and the high non-linearity inherent to DNNs, making it applicable to a wide variety of architectures, with a special focus on Residual Networks (ResNets) due to their practical importance. Unlike existing methods, our algorithm imposes minimal assumptions on the internal architecture of neural networks, thereby extending its usability to complex models. Theoretical analysis guarantees convergence, while extensive empirical evaluations—including optimization tests involving functions with multiple local minima—demonstrate the robustness of our algorithm in navigating non-convex response surfaces. The experimental results highlight the algorithm’s efficiency in accurately estimating DNN output ranges, even in scenarios characterized by high non-linearity and complex constraints.
\end{abstract}


\section{Introduction}
Unquestionably, in recent decades, Deep Neural Networks (DNNs) have been by far the most widely used tools to perform complex machine learning tasks. More recently, DNNs have been used in cyber-physical systems critical to public security and integrity; such as autonomous vehicle driving and air traffic systems. Therefore, it is of pressing interest to implement security verification systems for DNNs. One of the objectives in this line of interest is the verification of the maximum and minimum values assumed by a DNN, an objective commonly known as the range estimation problem, see \cite{dutta2018output,wang2018formal}. This interest in estimating the range assumed by a DNN responds to the objective of diagnosing and validating the previously executed training. However, the relationships established between the inputs and outputs of a DNN are highly non-linear and complex, making it difficult to understand with existing tools today. Due to this inability, DNNs are commonly referred to as black boxes. This nature of DNN makes the range estimation problem particularly challenging because there is no geometric information about the response surface generated by a DNN. For example, if local geometric information about the generated surface by a DNN was obtained, such as the gradient vector and the Hessian matrix at each point, the problem could be addressed with conventional nonlinear programming techniques. However, in a DNN it is only possible to obtain point information about the estimated response, without any local knowledge around that point. These facts, added to the high nonlinearity of a DNN, make the range estimation problem nontrivial.  In particular, this problem is even more challenging when dealing with a DNN, in contrast to other types of neural network, since its multiple layers increase the complexity of the problem. In this paper, we present a simple algorithm based on global optimization techniques, initially motivated by the classical Simulated Annealing (SA) \cite{kirkpatrick1983optimization}, to solve the range estimation problem for a wide spectrum of neural networks, and in particular DNNs. In contrast to classical SA, our algorithm considers the existence of a restricted domain for searching for optimal points, which eventually corresponds to the domain of the training data of the analyzed neural network. Our algorithm does not make use of any information about the internal architecture of the analyzed DNN, which makes it applicable to a much larger spectrum than other proposals available in the literature, see for example \cite{dutta2018output,wang2018formal,katz2017reluplex}. Our algorithm considers restricted search spaces, which underlie the nature of the output range estimation problem of neural networks. Furthermore, we present results, both theoretical and empirical, that guarantee the convergence of our algorithm towards the optimal points, which leads to a good estimation of the output range.

    In line with the objective and motivation stated—estimating, in a reliable way, the extreme values that the network can produce under bounded inputs while treating it as a black box—we adopt a global search strategy on a finite domain that combines four simple and complementary ideas. First, symmetric random proposals are generated around the current point to explore without directional bias. Second, a cyclic reflection is applied at the domain boundaries: any proposal that leaves the range “bounces” back in by symmetry, preventing artificial accumulation near the boundary. Third, a probabilistic acceptance criterion balances exploration and exploitation and, in a controlled way, allows accepting moves that do not immediately improve in order to escape local optima. Fourth, the intensity of exploration is reduced gradually so that the search progressively concentrates on the most promising regions. These elements are articulated into a single update mechanism: at each iteration a candidate is proposed (with reflection if needed), the network output is evaluated at that point, and acceptance or rejection is decided; if accepted, the state is updated, and in all cases the best value found and its location are recorded. The formal details and assumptions are presented in the methodology Section \ref{SA}. The rationale for this dynamics is that it combines broad, unbiased coverage with progressive focusing. Symmetric proposals together with cyclic reflection support a balanced exploration of the admissible space, preventing the trajectory from sticking to the boundaries and encouraging visits to distinct regions of the domain. The probabilistic acceptance criterion acts as a safeguard against entrapment: early on it facilitates barrier crossing and exits from local optima; later, as the exploration intensity is gradually reduced, it acts as an adaptive “zoom” that concentrates sampling where the evidence is stronger. Moreover, by always retaining the best value observed, the procedure does not lose progress: even while exploring alternatives, the best estimate of the extreme improves monotonically. In a black-box setting with bounded inputs, this sequence makes visits to neighborhoods of the true extrema very likely and, under mild continuity, yields estimates close to the extreme values of interest.

    Our methodology for output range estimation in deep neural networks shares similarities with global optimization methods used in previous studies. For example, Gromov et al. \cite{Gromov2024} and Kovaleva \& Smirnov \cite{Kovaleva2024} explore optimization techniques for finding global extrema in complex systems, including neural networks. These methods, such as Simulated Annealing (SA) and memetic algorithms, are effective for general optimization problems. However, while these studies focus on optimizing various types of systems, our work introduces a novel approach by applying SA with reflective boundary conditions specifically designed for deep neural networks with constrained input spaces. This reflects a need in the literature for methods that preserve boundaries while optimizing over complex, high-dimensional spaces. Furthermore, previous works like the one by van Laarhoven \& Aarts \cite{LaarhovenAarts1987} and Nourani \& Andresen \cite{NouraniAndresen1998} provide strong theoretical foundations for SA in global optimization. Our approach complements these methods by maintaining asymptotic convergence guarantees while incorporating constraints imposed by the architecture of deep neural networks. Additionally, recent studies such as the one by Gromov et al. \cite{Gromov2024} on Bayesian optimization and the study of evolutionary algorithms for global optimization by Kovaleva \& Smirnov \cite{Kovaleva2024} further emphasize the importance of global extrema search in machine learning contexts. Our method extends these approaches by directly addressing the challenges posed by the constrained domains in neural network outputs, offering a practical solution for deep learning models. Finally, optimization methods and extremal problems in neural network training, such as those explored by Gromov et al. \cite{Gromov2024}, are directly addressed by our technique. While previous studies have applied evolutionary strategies and optimization algorithms to explore global minima, our approach with reflective boundaries adds a layer of robustness by ensuring that the optimization process remains within feasible input domains, thus overcoming the limitations of traditional methods in constrained environments.
 
\textbf{Outline}. In the section ``Residual Neural Networks" we make a mathematical description of a very particular family of deep neural networks called Residual Networks. Although our methodology works for any particular family, we focus on these neural networks given their relevance in applications and for illustrative purposes of the use of our algorithm. In the section ``Output Range Analysis Problem", we specifically define the output range analysis problem for neural networks, mentioning its characteristics, limitations, and challenges. In the section titled ``Simulated Annealing with Boundary Conditions" we develop an exhaustive treatment of Simulated Annealing for limited domains, adapting its properties and guaranteeing its results for our purposes. In the section titled ``Algorithm Derivation", from the above, we derive an algorithm to solve the output range analysis problem. Finally, in the section ``Experimental Evaluation" we present a part of the experimental evaluations that we have executed to numerically guarantee our results and the performance of our proposed algorithm.

\section{Related Work}

 Recently, various methodologies have been proposed for the analysis of the output range of DNNs. \cite{dutta2018output} propose a novel approach to estimate the output ranges of these models given specific input sets. Their novel algorithm utilizes local search techniques and linear programming to efficiently compute the maximum and minimum values in a DNN over the input set. The algorithm repeatedly applies local gradient descent to identify and eliminate local minima of the neural network's output function. Once local optima are identified, the final global optimum is validated through a mixed integer programming model. The authors consider the algorithm's computational efficiency while still maintaining the accuracy of the output range estimation. However, common issues like scalability and the potential complexity of a DNN are not fully addressed. Furthermore, the methodology makes specific assumptions about the neural network, which means that the methodology can only be applied to a rather restricted class of neural networks.

On the other hand, autonomous vehicles, collision avoidance systems, and others are part of a real-world security-critical domain, making it important to check the security properties of DNNs. These security properties guarantee that a DNN should have to ensure it operates safely and reliably against potential threats. There are frameworks for verifying the security properties of DNN that use interval arithmetic \cite{wang2018formal,katz2017reluplex,carlini2018provably}. \cite{wang2018formal} shows us a framework called ReluVal, which employs interval arithmetic to represent the ranges of inputs and outputs. For each input and operation in a DNN, the interval calculation is performed to explore the output range of the network. By analyzing these intervals, the framework can detect inputs that lead to incorrect classifications or outputs, thereby exposing potential weaknesses in the network. The experiments demonstrate the efficiency  of ReluVal in verifying security properties, outperforming other frameworks. Nevertheless, it may still struggle with scalability when applied to extensive networks or complex architectures. Furthermore, the framework is focused on specific types of adversarial attacks, leaving uncertainty regarding its performance on different architectures.

Other kinds of verification of DNNs are required: safety and robustness. \cite{tran2020nnv} presents the Neural Network Verification (NNV) software tool. The tool applies a collection of reachability algorithms and the use of a variety of set representations, such as polyhedra, star sets, zonotopes, and abstract-domain representations. Also, NNV can handle different types of neural networks and system models, allowing it to support both exact and over-approximate analysis. However, representations as polyhedra are limited by scalability and may not work effectively, as they involve costly operations, leading to a conservative and less practical reachability analysis in larger systems. Moreover, all experiments focus on specific exercises, which limits the exploration of new and complex DNN architectures. Then, \cite{liu2021algorithms} provide a comprehensive overview of methods developed to formally verify the safety and robustness of DNNs. The authors discuss different algorithms that incorporate reachability analysis, optimization, and search techniques to verify the appropriate configuration of the DNNs for the input-output properties across their entire input space. These algorithms are classified into different categories, such as layer-by-layer, reachability analysis, optimization-based, and combined search and verification methods. However, after exhaustive verification of these algorithms, the work concludes with some algorithms that present, like the other works, problems with the network scale. Finally, \cite{huang2020divide} analyze the output range of a neural network using a convex polygonal relaxation (over-approximation) of the activation functions to manage nonlinearity, allowing the problem to be formulated as a mixed-integer linear program (MILP). However, a key limitation is the addition of more integer variables in the MILP, which increases computational complexity, particularly in DNNs or those with more complex architectures.
  
In general, these proposals described above make restrictive assumptions about the network architecture which limits the use and application of these methodologies. In contrast, our proposal does not make any assumptions about the internal architecture of the evaluated neural network, which makes our algorithm widely applicable in output rank analysis for neural networks, and given the greater complexity, even more so when it comes to DNNs. Given these features of our algorithm, added to this are the theoretical and empirical guarantees that we present, and its simple and easy implementation, we believe that it is a relevant and practical contribution to the objectives in this application area.

\section{Residual Neural Networks}\label{Residual NN}

As mentioned above, since we make no assumptions about the internal architecture of the analyzed neural network, our algorithm works for a wide spectrum of neural networks. However, without loss of generality, we focus on  Residual Neural Networks given their relevance in applications and for illustrative purposes of the use of our algorithm. The Residual Networks (or ResNets) is a kind of DNN that was introduced in \cite{he2016deep}, for which in this chapter we show a mathematical description. Let $x\in\mathbb{R}^d$ denote the inputs and $y\in\mathbb{R}$ be the output in a supervised learning problem. In this learning context, the objective is the approximation of a function $f: x\in\mathbb{R}^d \mapsto y\in\mathbb{R}$ using a ResNet, which we denote as $\mathcal{F}$. The basic unitary components of a ResNet are called residual blocks, which are organized in a consecutive number of layers that are linked through nonlinear functions. To obtain a mathematical description of the operations carried out within a ResNet, we define some notations. We consider a network with a source layer, a output layer and $L$ hidden layers, where the $l$-th layer contains $H_l$ neurons ($l$-th layer width), for $l=0,\dots, L+1$. Note that $H_0=d$ and $H_{L+1}=1$, corresponding to the feature inputs and  target outputs respectively. We denote the output vector of the $l$-th layer by $x^{(l)}\in\mathbb{R}^{H_l}$ which corresponds to the input of the next layer. We set $x^{(0)}=x\in\mathbb{R}^d$ and $x^{(L+1)}=y\in\mathbb{R}$. For $l=1,\dots, L$, the $i$-th neuron performs an affine transformation on that layers input $x^{(l-1)}$ followed by a non-linear transformation
\begin{equation}
x^{(l)}_i=\sigma\Bigg(\sum_{j=1}^{H_{l-1}} W_{ij}x^{(l-1)}_j+b_i^{(l)}\Bigg)+x^{(l-1)}_i,
\quad 1\leq i\leq H_{l}.
\end{equation}

where $W_{ij}$ and $b_i^{(l)}$ are respectively known as the weights and bias associated with $i$-th neuron of layer $l$, while the function $\sigma(\cdot)$ is known as the activation function.
\begin{definition}[ResNet] Let $W^{(l)}=[W_{ij}^{(l)}]\in\mathbb{R}^{H_{l-1}\times H_{l}}$ be the weight matrix and $b^{(l)}=[b_i^{(l)}]\in\mathbb{R}^{H_{l}}$ be the bias vector for layer $l$, then the operations within each layer are described by
\begin{equation}
\begin{split}
    x^{(l)} & = \sigma\Big(\mathcal{A}^{(l)}\big(x^{(l-1)}\big)\Big)+ x^{(l-1)}, \\[6pt]
    \mathcal{A}^{(l)}\big(x^{(l-1)}\big) & = W^{(l)}x^{(l-1)}+b^{(l)}.
\end{split}
\end{equation}
where $\sigma$ acts component-wise; that is, $\sigma(x_1,\dots,x_d):=(\sigma(x_1),\dots,\sigma(x_d))$. Thus, a ResNet with $L$ hidden layers $\mathcal{F}:\mathbb{R}^d \longrightarrow \mathbb{R}$ is mathematically defined as
\begin{equation}\label{NNdef}
\mathcal{F}(x):=\mathcal{A}^{(L+1)}\circ\sigma\circ\mathcal{A}^{(L)}\circ\sigma\circ\mathcal{A}^{(L-1)}\circ\cdots\circ\sigma\circ\mathcal{A}^{(1)}(x).
\end{equation}
\end{definition}
Although the activation function can, in principle, be chosen arbitrarily, there are three functions that have particularly proven to be useful in various applications; ReLU, Sigmoid, and Tanh, see \cite{lecun2015deep,goodfellow2016deep}. The parameters of the network is all the weights and biases $\{W^{(l)}, b^{(l)}\}_{l=1}^{L+1}\in\mathbb{R}^{N_p}$, where  $N_p=\sum\limits_{l=1}^{L+1}(H_{l-1}+1)H_{l}$ is the total number of parameters.

\paragraph{Residual blocks with general activations.}
Let $\sigma:\mathbb{R}\to\mathbb{R}$ be a Lipschitz activation (e.g., ReLU, $\tanh$, GELU).
To make the residual path explicit and avoid ambiguity between (2) and (3), we write
\begin{align*}
    x^{(0)} & = x, \\[6pt]
    x^{(l+1)} & = \sigma\!\big(A^{(l)}(x^{(l)})\big) + x^{(l)}, 
    \qquad l = 0,\ldots,L-1.
\end{align*}

and the network output
\[
\mathcal{F}(x) = \mathcal{A}^{(L+1)}\!\big(x^{(L)}\big).
\]
These expressions preserve the skip connection for any choice of $\sigma$. Note that our algorithm treats the network as a black box; the presence of residual connections is not required by the method or by the analysis. In this case, we display the residual formulation to adopt notation that is standard in modern architectures and to align with the models used in our experiments, ResNet-based surrogates of Ackley and Drop–Wave. The simulated-annealing procedure and the theoretical development, however, treat $\mathcal{F}$ as a black box; the same analysis applies verbatim to generic feedforward DNNs with affine layers and a Lipschitz activation.

 \section{Output Range Analysis Problem}\label{problem} 
In the rest of this paper, we consider a ResNet $\mathcal{F}$ with inputs $x\in\mathbb{R}^d$ and output $y\in\mathbb{R}$ as defined \eqref{NNdef}. We assume that the network parameters of the neural network were optimally estimated following standard training procedures such as those shown in \cite{bishop2006pattern,murphy2012machine}.
\begin{definition}[Range Estimation Problem]\label{main-problem} The problem is defined as follows:
\begin{itemize}
    \item INPUTS: : A trained neural network (for example a ResNet $\mathcal{F}$) and input constraints $Ax\leq b$ which generates a feasible set $E$, i.e., $E=\{x\in\mathbb{R}^d: Ax\leq b\}$.

     \item OUTPUT: An interval $[\mathcal{F}_{min}, \mathcal{F}_{max}]$ such that $\mathcal{F}(x)\in [\mathcal{F}_{min}, \mathcal{F}_{max}]$, i.e., $[\mathcal{F}_{min}, \mathcal{F}_{max}]$ contains the range of $\mathcal{F}$ over inputs $x\in E$. Moreover, the optimal points $\{x_{min},x_{max}\}\in E$  such that $\mathcal{F}(x_{min})=\mathcal{F}_{min}$ and $\mathcal{F}(x_{max})=\mathcal{F}_{max}$.
\end{itemize}   
\end{definition}
\begin{remark}
Due to the applications, we focus just on hypercubes of $\mathbb{R}^d$ as feasible set, that is; sets of the type $E=[l_1, u_1]\times \cdots\times[l_d, u_d]\subset \mathbb{R}^d$, where $[l_j, u_j]$ are intervals in $\mathbb{R}$. Note that $E$ is the restricted search domain in which we want to diagnose and validate the analyzed neural network.   
\end{remark}

As mentioned above, this problem is particularly challenging given that there is generally almost total ignorance about the surfaces generated by a DNN. That is, there is no information on the functional form of the mapping $\mathcal{F}$. Therefore, there is no local geometric information such as the local slopes and curvatures of these surfaces, making the use of traditional nonlinear programming methodologies impossible. Furthermore, the surfaces generated by this type of neural network are strictly non-convex, with many minima. Therefore, it requires thinking about global optimization methods, methods such as SA for which it is necessary to incorporate the restricted search domain $E$.

\section{Simulated Annealing with Boundary Conditions}\label{SA}
Our objective in this session is to incorporate the domain $E$, coming from the range estimation problem, into the classical SA, and at the same time establish theoretical guarantees of the optimization process of the neural network $\mathcal{F}$. To this end, our range estimation problem can be summarized in an optimization problem of the form 
\begin{equation}
   \min_{x\in E}  \mathcal{F},
\end{equation}
where $\mathcal{F}$ is a ResNet and $E$ is a hypercubes of $\mathbb{R}^d$, which we denote by the pair $(\mathcal{F}, E)$. The  main problem now is to find a point $x_{\min}\in E$ such tha $\mathcal{F}(x_{\min})$ is global minimal on $E$. The $\max_{x\in E}  \mathcal{F}$ is equivalent to $\min_{x\in E}  -\mathcal{F}$, which is why we can, without loss of generality, only talk about “Minimization” throughout this paper. We denote by $M_{\mathcal{F}}=\{x_{\min}\in E:\mathcal{F}(x_{\min})\leq\mathcal{F}(x),\,\textrm{for all}\,\,x\in E\}$ the minimum set of $\mathcal{F}$. From Equation \eqref{NNdef}, we know that $\mathcal{F}$ is a continuous function, and added to the fact that $E$ is a compact set, we know that $M_{\mathcal{F}}$ is a non-empty set. For simplicity in notation we set $\mathcal{F}_{\min}=\mathcal{F}(x_{\min})$. Next, we will incorporate the search domain by establishing reflective band conditions on $E$.
\begin{definition}[Cyclic Reflection]
    Let $E=[l_1, u_1]\times \cdots\times[l_d, u_d]\subset \mathbb{R}^d$, where $[l_j, u_j]$ are intervals in $\mathbb{R}$. We denote by $\mathcal{R
    }:\mathbb{R}^d\longrightarrow E$, the different combinations of the reflections on $E$, that is, $\mathcal{R}(y)=\big(\mathcal{R}(y_1),\dots,\mathcal{R}(y_d) \big)$, where
    \begin{equation}
    \resizebox{.9\linewidth}{!}{$
\mathcal{R}(y_j) = 
\begin{cases}
l_j + \big[(y_j - l_j) \mod (u_j - l_j)\big], & \text{if } (y_j - l_j) \mod \big(2(u_j - l_j)\big) \leq u_j - l_j, \\[10pt]
u_j - \big[(y_j - l_j) \mod (u_j - l_j)\big], & \text{if } (y_j - l_j) \mod \big(2(u_j - l_j)\big) > u_j - l_j.
\end{cases}
$}
    \end{equation}
\end{definition}

\begin{remark}[Intuitive role of the reflection step]
The range estimation problem is inherently posed on bounded input domains—here, a hypercube $E$—and any
search scheme must therefore keep all iterates within $E$. We employ a \emph{cyclic reflection} that returns any out-of-box proposal to $E$ via the function $\mathcal{R}$, so that local exploration near the boundary $\partial E$ resembles that of interior points and boundary-driven bias is avoided. Crucially, this boundary mechanism is designed to
preserve the assumptions used in our theoretical analysis and ensure the asymptotic convergence of the simulated-annealing
chain to the set of global extrema on $E$, see Theorem \ref{mainresult} and Corollary \ref{Corollary}. Compared with the general-state-space framework of \cite{haario1991simulated}, which is implicitly compatible with boundary handling but does not instantiate an explicit, algorithmic reflection tailored to box-constrained output-range estimation, our contribution makes this reflection function explicit and operational for the present objective.
\end{remark}

\begin{definition}[Generating Distribution]\label{Q}
 We say that ${Q}$ is a generating distribution, with Gaussian density function $\rho:E\times\mathbb{R}^d \longrightarrow \mathbb{R}^+$ and with reflective boundary conditions on $E$, if ${Q}$ is defined as
 \begin{equation}
     {Q}(x,B):=\int\limits_{\{y\in \mathbb{R}^d: \mathcal{R}(y)\in B\}}  \rho(x,y)\,\mathrm{d}y,
 \end{equation}
where
 \begin{equation}
    \rho(x,y)=\frac{1}{(2 \pi \sigma)^{d/2}} \exp \left( -\frac{\|x - y\|^2}{2\sigma} \right)
 \end{equation}
for $x\in E$, $B \subset\mathfrak{B}(E)$ and $y\in \mathbb{R}^d$. Here $\sigma$ is a positive parameter and $\mathfrak{B}(E)$ denote the Borel $\sigma$ -algebra on the state space $E$.
\end{definition}

\begin{definition}[Acceptance Probability]
  Given $\mathcal{F}$ and a number $T>0$, the acceptance probability $q_{T}:E\times E \longrightarrow \mathbb{R}^+$ is defined as 
  \begin{equation}
q_{T}(x, y) = e^{\frac{1}{T} \min\{0, \mathcal{F}(x) - \mathcal{F}(y)\}},
\end{equation}
where $T$ is called temperature parameter.
\end{definition}

\begin{definition}[Simulated Annealing Process]\label{X}
    Let $(\mathcal{F}, E)$ be a global minimization problem, $(Q_i)_{i\in\mathbb{N}}$ be a sequence of generating distributions with reflective boundary conditions, $\{T_i\}_{i\in\mathbb{N}}\downarrow 0$ be a sequence of temperature parameters, and $(q_{T_i})_{i\in\mathbb{N}}$ be a sequence of acceptance probabilities. A simulated annealing process with reflective boundary conditions on $E$ is the non-homogeneous Markov process $(X_i)_{i\in\mathbb{N}}$, defined on a probability space $(\Omega, \mathcal{A}, \mathbb{P})$, with state-space $(E,\mathfrak{B})$ and transition kernel $(P_{i})_{i\in\mathbb{N}}$ defined by
   {\begin{equation}\label{Kernel}
   \resizebox{.9\linewidth}{!}{$
P_{i}(x,B) =
\begin{cases}
\int\limits_{ B} q_{T_i}(x,y) \, Q_i(x,y)\, \mathrm{d}y & \text{for } x \notin B, \\
\int\limits_{ B} q_{T_i}(x,y) \,Q_i(x,y)\, \mathrm{d}y + \Big(1 - \int\limits_{E} q_{T_i}(x,y) \,Q_i(x,y)\, \mathrm{d}y \Big) & \text{for } x \in B,
\end{cases}
$}
\end{equation}}
where $x\in E$ and $B \subset\mathfrak{B}$.
\end{definition}

{This definition aligns with the framework established in Definition 2.4 of Haario and Saksman \cite{haario1991simulated}, which introduces the simulated annealing process in a general state space. Their work provides a rigorous foundation for the transition kernels and the Markov process in this context.}

Given the transition kernel in \eqref{Kernel}, we defined the distributions of Markov process $(X_i)_{i\in\mathbb{N}}$ as $\mu_i(\mathrm{d}x):=\mathbb{P}(X_i\in \mathrm{d}x)$. Consequently, we have that $\mu_i=\mu_{i-1}P_{i}$ for $i\geq 1$, where $\mu_0$ is an arbitrary initialization distribution on $E$.

\begin{theorem}\label{mainresult} Given the minimization problem  $(\mathcal{F}, E)$, let $(X_i)_{i\in\mathbb{N}}$ a simulated annealing process with reflective boundary conditions on $E$. Then, for each $i\in \mathbb{N}$, the operator $P_i$ has the equilibrium distribution $\pi_i$ given by
\begin{equation}\label{resul1}
\pi_i(y) = C_i \exp \left(-\,\frac{\mathcal{F}(y)-\mathcal{F}_{\min}}{T_i} \right),
\end{equation}
where $C_i$ is the normalization constant and $\mathcal{F}_{\min}\in M_{\mathcal{F}}$. Moreover, if we set $T_i=T_0\,\delta^i$, for some $|\delta|<1$ and $T_0$ arbitrary, then we have that
\begin{equation}\label{resul2}
\lim_{i\to \infty}||\mu_i-\pi_i||_{TV}=0,
\end{equation}  
where $||\cdot||_{TV}$ is the total variation norm.
\end{theorem}
\begin{remark}The total variation distance between two probability measures \(P\) and \(Q\) defined on a measurable space \((\Omega, \mathcal{F})\) is given by
\[
||P-Q||_{TV} = \sup_{A \in \mathcal{F}} |P(A) - Q(A)|.
\]
This measures the largest absolute difference between the probabilities that the two distributions assign to the same event. In the context of our problem, we apply the total variation norm to measure the difference in the distributions of the output range across iterations.
\end{remark}

\begin{proof}
    By construction, ${Q}$ is symmetric, that is
    \begin{equation}
        \int\limits_{B} {Q}(x,y)\,\mathrm{d}y= \int\limits_{B} {Q}(y,x)\,\mathrm{d}x,
    \end{equation}
    for $B \subset\mathfrak{B}$. Furthermore, $\mathcal{F}$ is a continuous function, and added to the fact that $E$ is a compact set, then $M_{\mathcal{F}}$ is a non-empty set and $\mathcal{F}:E \longrightarrow \mathbb{R}$ is uniform continuous. Given these facts, the results in \eqref{resul1} and \eqref{resul2} are direct applications of Theorem 5.1 and Theorem 6.5 in \cite{haario1991simulated}.
\end{proof}

\begin{corollary}\label{Corollary}
    Suppose that the conditions of Theorem \ref{mainresult} are fulfilled, then 
    \begin{equation}
        \mathcal{F}(X_i)\longrightarrow \mathcal{F}_{\min}\quad \textrm{as}\quad i\longrightarrow \infty\quad \textrm{in probability},
    \end{equation}
    for some $\mathcal{F}_{\min}\in M_{\mathcal{F}}$.
\end{corollary}
\begin{proof}
The proof is a slight modification of  Corollary 5.4 in \cite{haario1991simulated}.
\end{proof}
As will be seen later, Corollary \ref{Corollary} is the mechanism that establishes the theoretical guarantees that our algorithm converges to the global minimum point.

\section{Algorithm Derivation}\label{Derivation}
From Theorem \ref{mainresult} and its Corollary \ref{Corollary}, a simple algorithm can easily be derived to solve the range estimation problem $(\mathcal{F}, E)$. The main idea is to generate a Markov process $(X_i)_{i\in\mathbb{N}}$ taking values on $E$, as described in Definition \ref{X}, for which its initial state is selected using an arbitrary measure $\mu_{0}$. The following states of the Markov process are selected according to two stages; In the first stage, a new state generated by $Q$ is proposed, and subsequently in the second stage, it is decided to accept or reject that new state according to the probability of acceptance $q$. If the proposed state is outside the domain $E$ we use reflection $\mathcal{R}$ to exchange it for another equally probable state that is inside the domain $E$. It is important to note that the generation of new proposed states using $Q$ is done through $\rho$. Both stages, proposing and accepting new states, are executed repeatedly $N$ times for each temperature level $T$. Theorem \ref{mainresult} guarantees that if we execute this two-stage local recursive process the Markov process $(X_i)_{i\in\mathbb{N}}$ will converge towards the minimal state corresponding to the value $\mathcal{F }_{\min}$, this occurs for each set temperature level. Finally, Corollary \ref{Corollary} guarantees that if we execute the stages described above for a sequence of temperatures $\{T_i\}_{i\in\mathbb{N}}$ that slowly decreases towards zero, then the Markov process $(X_i)_{i\in\mathbb{N}}$ will converge to the minimal state corresponding to the global minimum that solves our estimation problem. We summarize the entire procedure described above in pseudo-code format in Algorithm~\ref{alg:algorithm}. 

\paragraph{Convergence and cooling schedule.}
The theoretical result of the previous section establishes {asymptotic convergence} under the stated assumptions, compact box $E$, continuity of $\mathcal{F}$, and reflective update via the function $\mathcal{R}$. We do {not} derive finite-time rates for two complementary reasons. First, in simulated annealing such rates depend on {landscape features} of $\mathcal{F}$ that, in our black-box setting, are neither observable nor reliably estimable without imposing strong assumptions or accessing internal model information; forcing those assumptions would substantially broaden the scope without adding direct operational value to output-range estimation. Second, the aim of the paper is operational: to provide an effective procedure with asymptotic guarantees and clear usage guidance. In practice, the community often adopts \emph{geometric cooling} $T_{k+1}=\delta\,T_k$ as a pragmatic compromise that accelerates progress empirically, while classical global guarantees are associated with \emph{logarithmic} schedules, see, e.g., \cite{Hajek1988}. The practical guidelines we follow—choosing $\delta$ and tuning the proposal scale to maintain a moderate acceptance rate—are consistent with the classical SA literature and empirical comparisons \cite{LaarhovenAarts1987,AartsKorst1989,NouraniAndresen1998,LundyMees1986}. Implementation choices (target acceptance range, proposal scaling, and stopping criterion) are documented in the experimental evaluations section, in line with this framework.

\begin{algorithm}
\caption{Minimum Output Search for a Neural Network}
\label{alg:algorithm}
\begin{algorithmic}[1]
\Statex \textbf{Input:} The minimization problem $(\mathcal{F}, E)$
\Statex \textbf{Initialize} $i \gets 0$, $x\sim\mu_0$, $T_0 \gets T_{\max}$ and  $\mathcal{F_{\min}} \gets \mathcal{F}(x)$
\Statex \textbf{Output}: $\mathcal{F}_{\min}$
\While{$T_i>T_{\min}$}
    \For{$k = 1$ to $N$}
        \State Generate $y \sim \rho(x,y)$ 
        \State Reflect $y\gets \mathcal{R}(y)$ 
        \State Set $\Delta\mathcal{F}\gets \mathcal{F}(y)-\mathcal{F}(x)$
        
                     \If{$\Delta\mathcal{F} < 0$} 
    \State $q(x, y) = 1.0$
\Else
   \State $q(x, y) =\min\Big\{1, e^{-\frac{\Delta\mathcal{F}}{T_i}}\Big\}$
\EndIf 
   \State Sample $U\sim \textrm{Unif}\,(0,1)$ 
    \If{$U\leq q(x, y)$}
        \State Update $x \gets y$ 
    \EndIf

    \If{$\mathcal{F}(x) < \mathcal{F}_{\min}$}
        \State $\mathcal{F}_{\min}\gets \mathcal{F}(x)$ 
    \EndIf
        
    \EndFor
    \State Update $T_i=T_{i-1}\,\delta^i$
    \State Update $i \gets i+1$
    
\EndWhile

\State \textbf{Output:} $\mathcal{F}_{\min}$
\end{algorithmic}
\end{algorithm}

\section{Experimental Evaluations}\label{Experimentacion}

Based on the above constructions, our algorithm can be used to estimate the maxima and minima assumed by neural networks that were trained using an available database. It is important to mention that the database provided for training the neural network does not necessarily contain the maxima and minima assumed by this neural network. These maxima and minima correspond to the optimal points of the continuous response surface generated by the neural network that arises from the adjustment of the discrete observations contained in the training database. The response surface generated by a neural network can be viewed as a smoothing of the data. But, as mentioned above, we do not know the explicit form of the function graph that corresponds to the response surface generated by the trained neural network. Therefore, our strategy to empirically evaluate our algorithm consists in proposing some explicit functional forms for which we know their optimal points, in this case their minimum points. We will use these explicit functions, for which we have all possible geometric information, including their minimum points, to generate a noisy discrete sample through uniform sampling in its domain. These discretized samples are used as training data for a ResNet. The main idea of our experimental evaluation strategy is that by applying this algorithm to the ResNet, trained with the discretized data, we can approximate the true global minima of the functions that were used to generate these discretized samples.

Following the empirical evaluation scheme proposed above, and to illustrate the use and operation of our algorithm, we will use a couple of commonly used functions for testing optimization algorithms. This functions are: Ackley Function, Drop-Wave Function and a function with multiple global minima.

\subsection{The Ackley Function} In mathematical optimization, the Ackley function is a non-convex function, with many local minima, which is often used as a performance test problem for optimization algorithms.This function on a 2-dimensional domain is defined by
\begin{equation}\label{AckleyFunction}
\resizebox{\linewidth}{!}{$
   f(x_1, x_2) = -20 \exp \left( -0.2 \sqrt{0.5(x_1^2 + x_2^2)} \right) - \exp \left( 0.5 (\cos 2\pi x_1 + \cos 2\pi x_2) \right) + e + 20,
   $}
\end{equation}
for which its global minimum point is $f(0,0)=0$, see Figure \ref{smooth-Ackley}. We use this function to obtain a noisy discretized sample $\mathcal{D}=\{(x_{1i},x_{2i},f_{i})\}_{i=1}^{m}\subset E$, where $E=[-4,4]\times[-4,4]$, see \ref{smooth-Ackley}. Furthermore, suppose we want to estimate the function \eqref{AckleyFunction} using only the discretized sample of points $\mathcal{D}$. To this end, we use a ResNet as described in the chapter ``Residual Neural Networks,'' with a total depth of $L=5$ residual blocks. The input layer maps a 2-dimensional input to $H_0=128$ neurons. The network contains 4 intermediate residual blocks, each with $H_l=256$ neurons for $l=1, \dots, 4$, followed by a fifth residual block that reduces the dimensionality back to $H_5=128$ neurons before the output layer. The output layer maps the final $128$ neurons to a single output. All layers in the network use the ReLU activation function. The model was initialized with the following setup: Mean Squared Error (MSE) was employed as the loss function, and the Adam optimizer \cite{kingma2017adam} was used smooth surface generated by ResNet for training with a learning rate of 0.001. The model was trained for 1000 epochs. The fit generated by these ResNet is quite good, resulting in a mean absolute error of $1.0070$ and a mean squared error of $3.0298$ from $1000$ uniform random points generated within the range $[-5,5]\times[-5,5]$. The surface generated by the neural network maintains all the properties of the original function, including its non-convex nature and the existence of many local minima, see Figure \ref{smooth-Ackley}. As mentioned above, we will use this generated response surface to find the global minimum of the Ackley function.

\begin{figure}[h]
    \centering
    \includegraphics[width=0.3\textwidth]{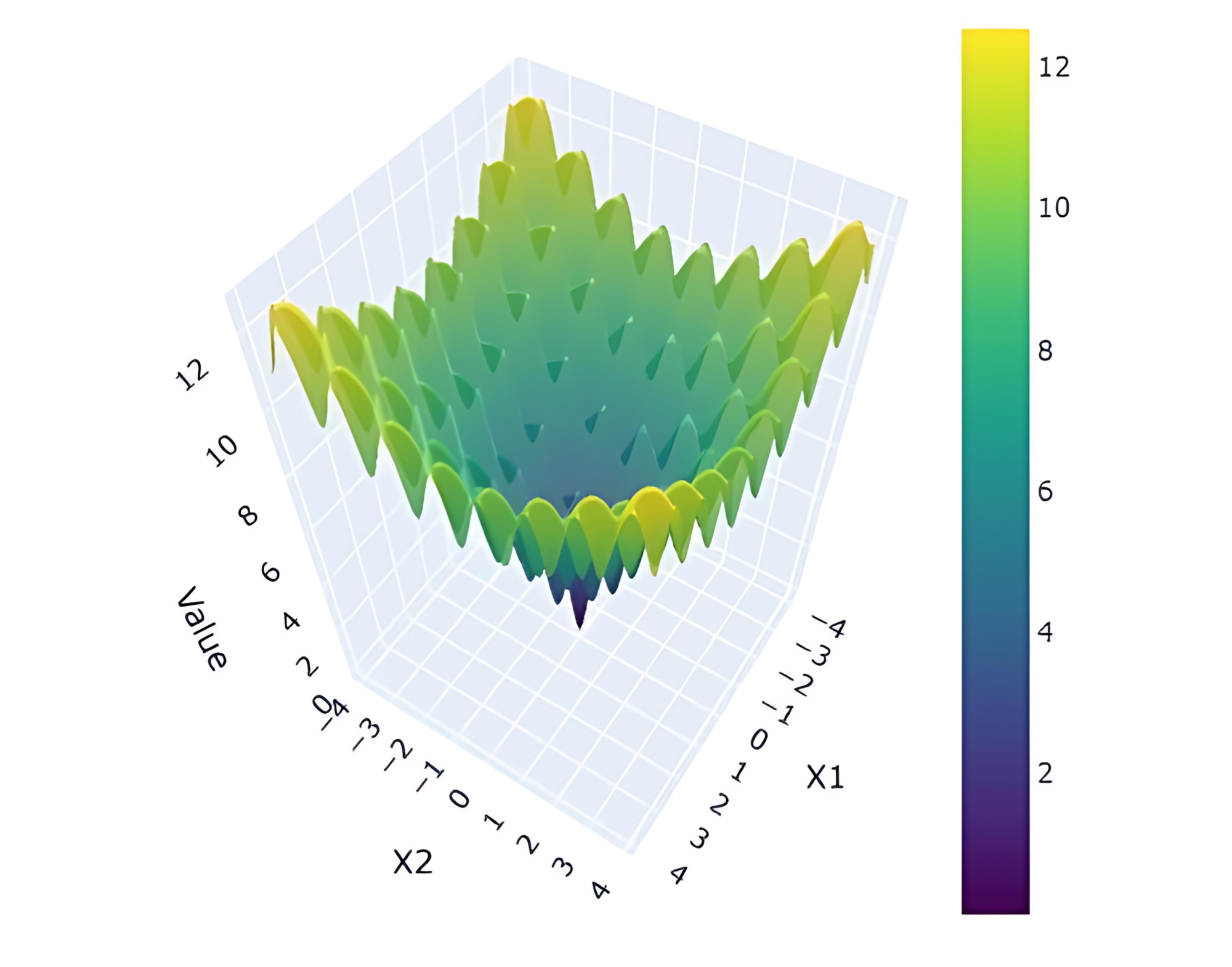}
    \subcaption{Ackley function in two dimensions.}

    \vspace{0.5cm}

    \includegraphics[width=0.3\textwidth]{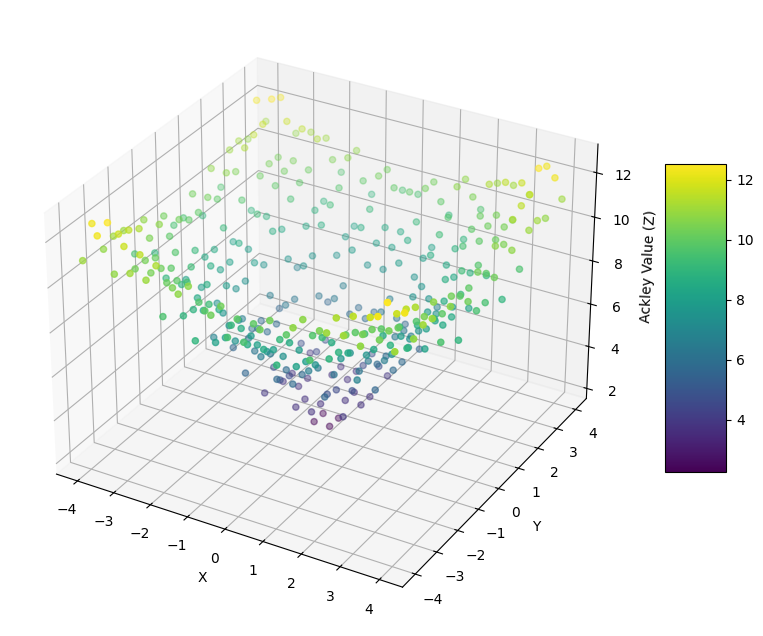}
    \subcaption{Discretized sample from the Ackley function}
    \hspace{0.5cm}
    \includegraphics[width=0.3\textwidth]{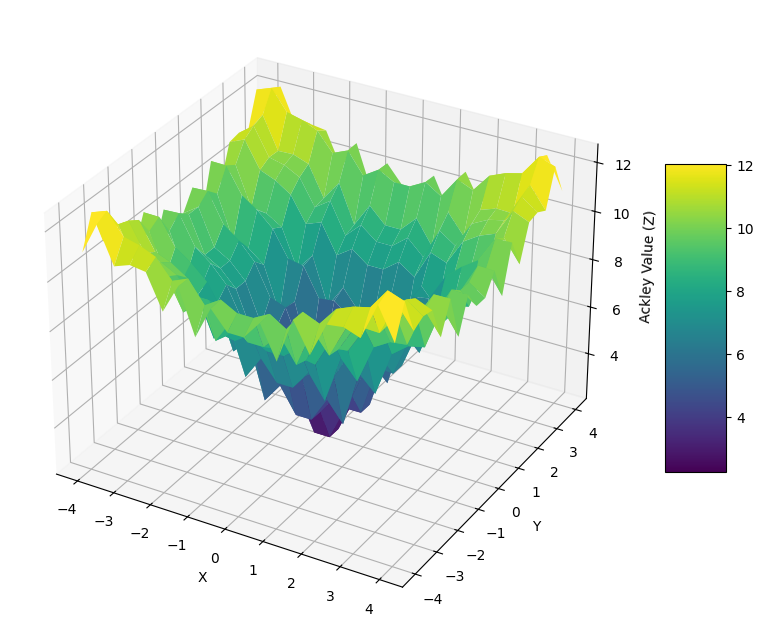}
    \subcaption{Response surface generated (smoothing) from the ResNet}

    \caption{Graphs of the experimental evaluation using the Ackley function.}
    \label{smooth-Ackley}
\end{figure}

The fact that the generated surface consists of too many local minima constitutes a suitable scenario where our methodology shows its greatest strength. Let us remember that when working with deep neural networks, this scenario is quite recurrent due to the high non-linearity and flexibility of the relationships established by this type of neural network. After training the ResNet $\mathcal{F}$, our main objective is to find the global minimum on the surface generated with the neural network within the domain $E$, denoted as $(E, \mathcal{F}(E))$. To this end, we apply our Algorithm ~\ref{alg:algorithm} to this case. In figure \ref{fig03}, we present the evolution of the values assumed by the Markov process generated through the algorithm and its tendency. We note the reasonably fast convergence towards the global minimum $\mathcal{F}_{\min}=0$. While it is true that there are fluctuations around the trajectory towards convergence, this is due to the highly fluctuating nature of the generated surface $(E, \mathcal{F}(E))$, but there is also a clear and consistent trend towards the global minimum $0$.

\begin{figure}[h]
    \centering
    \includegraphics[width=0.4\linewidth]{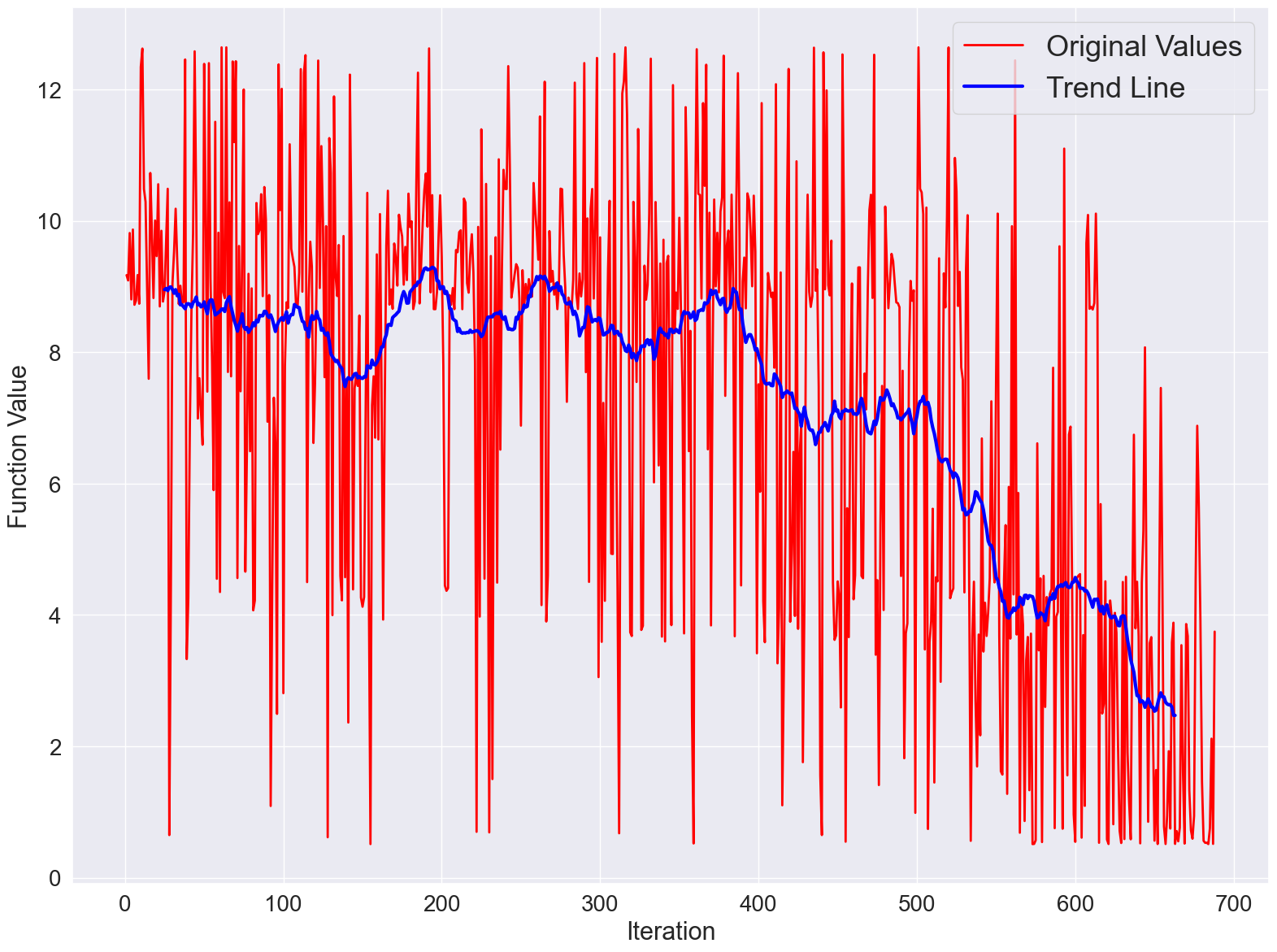}
    \caption{Red line: Sample path of Markov process 
    $(X_i)_{i\in\mathbb{N}}$ for the Ackley function. Blue line: fitted trend line.}
    \label{fig03}
\end{figure}

\subsubsection{Differences with the Classical Simulated Annealing Algorithm} 
As mentioned above, the fundamental difference between our proposed algorithm and the classical Simulated Annealing algorithm is the incorporation of a restricted search domain. This domain restriction is a priority need present in applications within the area of output range analysis for neural networks. As shown in Section \ref{SA}, we successfully incorporate bounded domains into the algorithm and further establish theoretical properties that guarantee the convergence of the algorithm to optimal points. In this subsection, ignoring the priority need for domain restrictions, for performance comparison purposes only, we apply the classical SA algorithm with a scheme analogous to Algorithm \ref{alg:algorithm} but without the incorporation of bounded domains. In Figure \ref{fig07} shows the extensive range and unbounded path traced by the classic Simulated Annealing algorithm. In iteration 125 a maximum value, $131.21$, is found to subsequently fluctuate to converge to the minimum. Therefore, the non-incorporation of a restricted domain causes the algorithm to experience large fluctuations before converging to the optimal points. But, as we mentioned before, this comparison is only didactic, since restricted domains are part of the nature of the problem. In many cases, as part of the diagnosis of the trained neural network, it is necessary to evaluate the optimal points assumed by the neural network in certain previously defined limited domains.

\begin{figure}[h]
    \centering
    \includegraphics[width=0.4\linewidth]{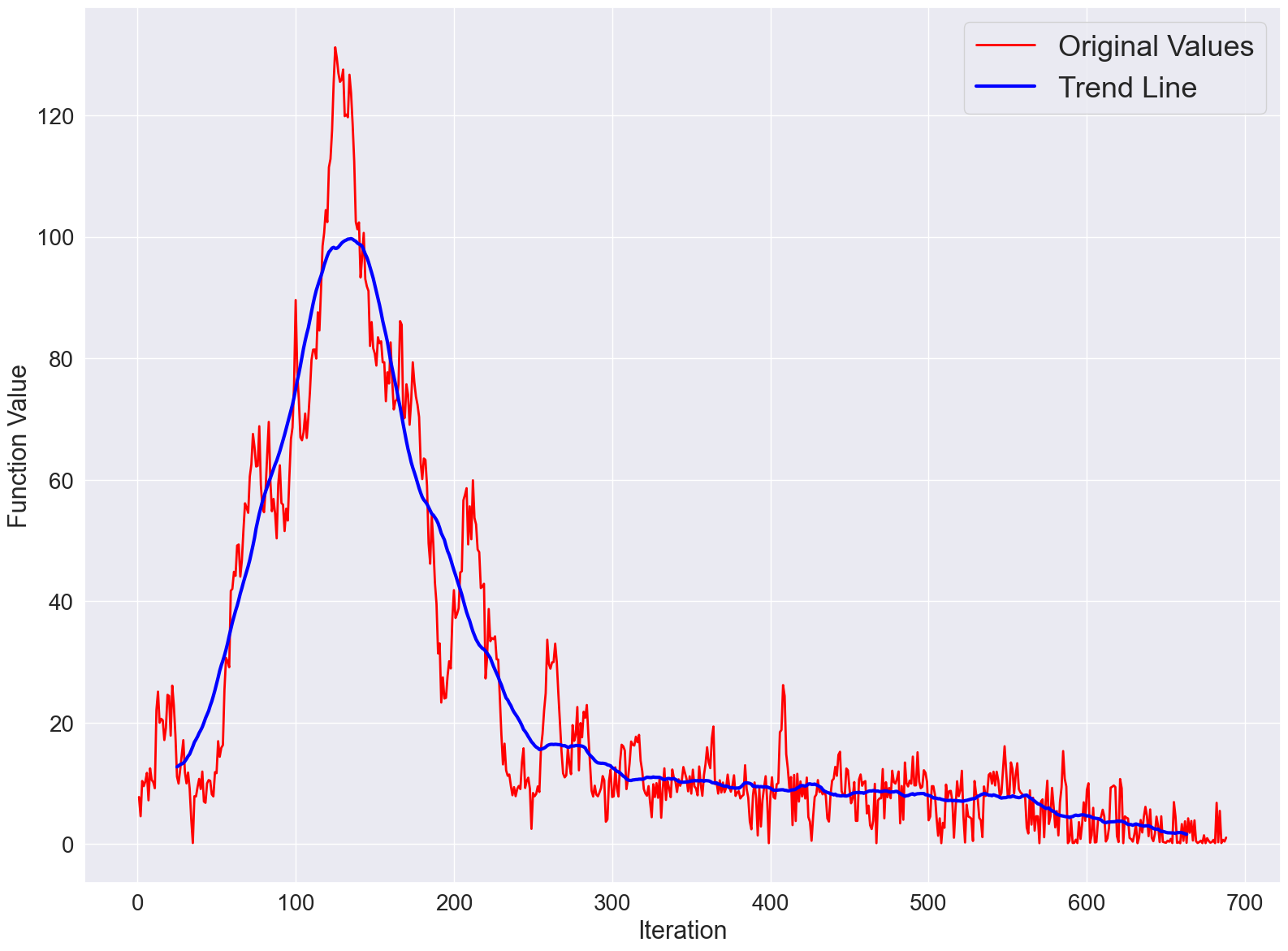}
    \caption{Red line: Sample path of the classical Simulated Annealing algorithm for the Ackley Function. Blue line: fitted trend line.}
    \label{fig07}
\end{figure}

\subsection{The Drop-Wave Function} 

This function is characterized by being multimodal, with many local minima, and highly complex; it is also non-convex. Like the previous function, it has a 2-dimensional domain and is defined by 
\begin{equation}\label{DropWaveFunction}
     f(x_1, x_2) = -\frac{1 + \cos\left(12 \sqrt{x_1^2 + x_2^2}\right)}{0.5(x_1^2 + x_2^2) + 2},
\end{equation}
for which its global minimum point is $f(0,0)=-1$, see Figure \ref{fig04}. In particular, the function is usually evaluated on the square $E=[-5.12, 5.12]\times[-5.12,5.12]$. Moreover, suppose we want to estimate the function \eqref{DropWaveFunction} using only a discretized sample of points $\mathcal{D}=\{(x_{1i},x_{2i},f_{i})\}_{i=1}^{m}\subset E$. For this experiment, we use a total depth of $L=7$ residual blocks. The input layer maps a 2-dimensional input to $H_0=128$ neurons. The network contains $L=7$ residual blocks with the following dimensions: $H_1 = H_2 = 256$, $H_3 = H_4 = H_5 = 512$, $H_6 = 256$, and $H_7 = 128$. Finally, The output layer has $128$ neurons. All layers in the network use the ReLU activation function
  The model was configured with the following setup: The loss function utilized was Mean Squared Error (MSE). Training was performed using the Adam optimizer \cite{kingma2017adam} with a learning rate of $0.001$ and was performed over 1000 epochs. The fit generated by these ResNets is strong, with a mean absolute error of $0.0221$ and a mean squared error of $0.0010$, based on 1500 points generated within the range $[-5.12, 5.12]\times[-5.12,5.12]$. The surface generated by the neural network maintains all the properties of the original function, including its non-convex nature and the existence of many local minima, see Figure \ref{fig05}.
\begin{figure}[h]
\centering
    \includegraphics[width=0.4\linewidth]{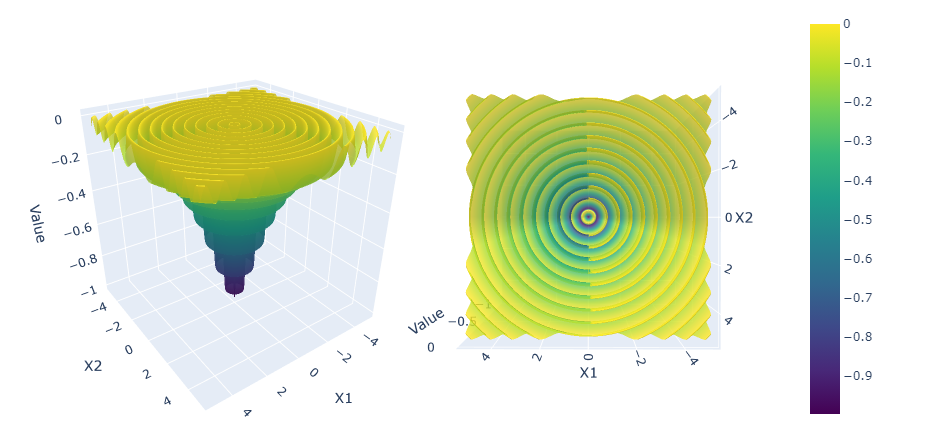}
    \caption{Drop-Wave function of two variables. Left: Perspective view; Right: Top View.}
    \label{fig04}
\end{figure}

\begin{figure}[h]
    \centering
    \includegraphics[width=0.4\linewidth]{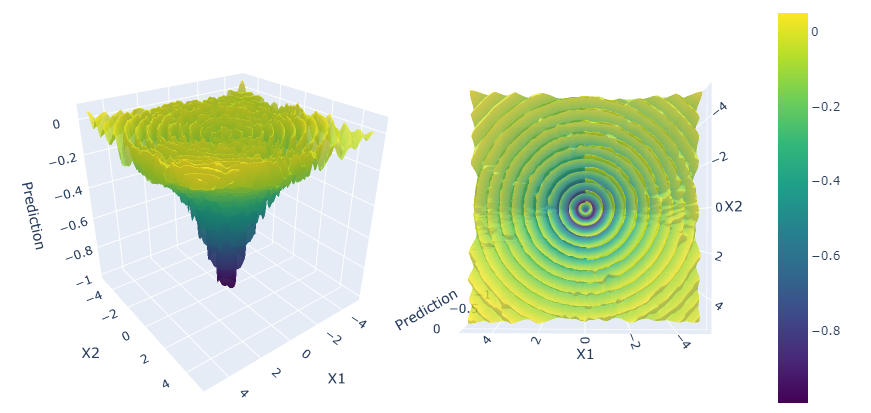}
    \caption{Drop-Wave surface generated by the ResNet $(E, \mathcal{F}(E))$. Left: Perspective view; Right: Top View.}
    \label{fig05}
\end{figure}

\begin{figure}[h]
    \centering
    \includegraphics[width=0.4\linewidth]{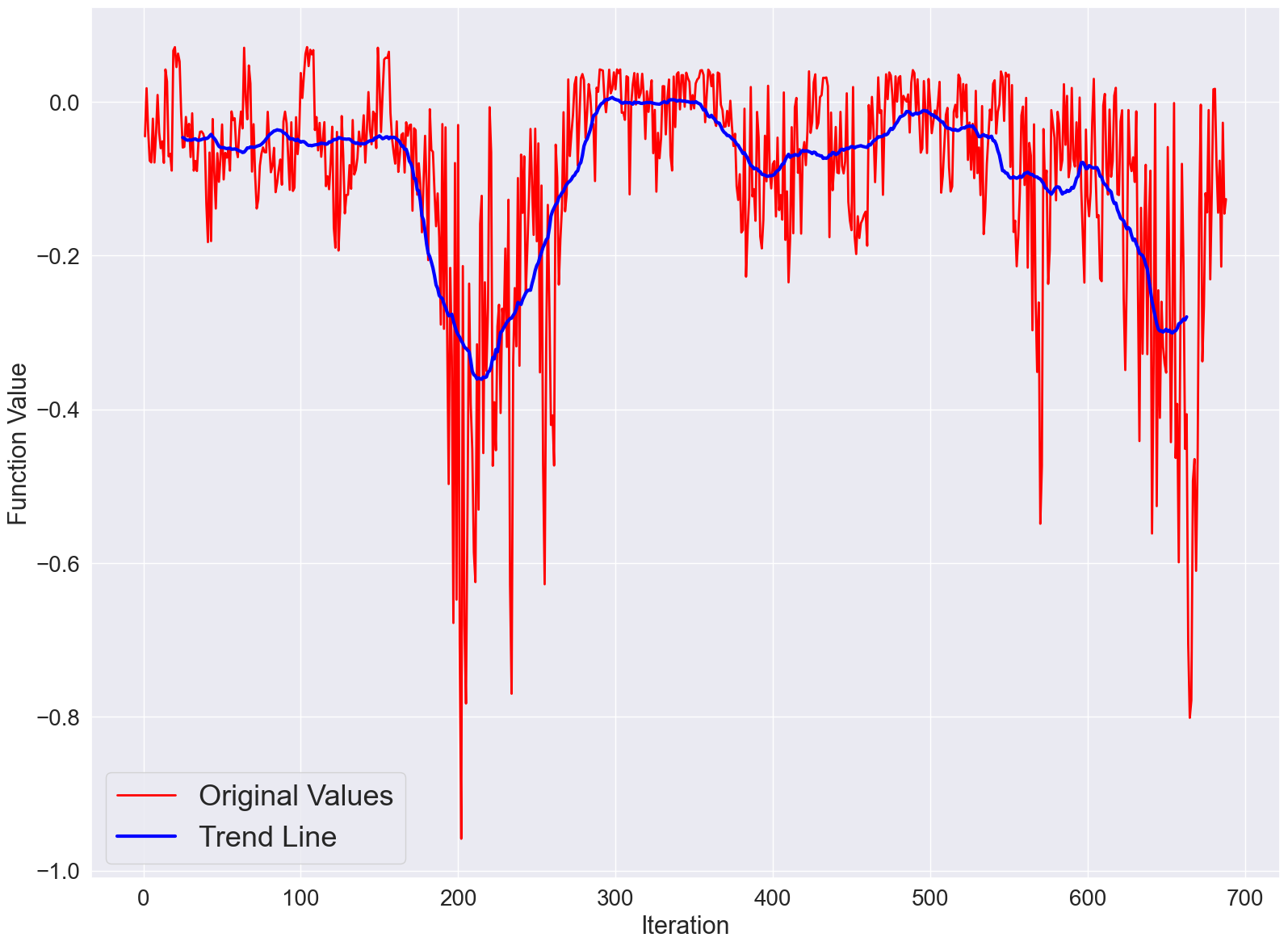}
    \caption{Red line: Sample path of Markov process $(X_i)_{i\in\mathbb{N}}$ for the Drop-Wave function. Blue line: fitted trend line.}
    \label{fig06}
\end{figure}

This second function is highly complex, and its experiment demonstrates the power of our method. Despite the presence of numerous local minima, our method successfully approximates to the global minimum. This scenario is crucial for solving highly complex problems, which are common in physics and other fields \cite{Bezhko_2021,physic_plasma}. We applied our Algorithm ~\ref{alg:algorithm} and obtained the results shown in Figure \ref{fig06}. The evolution of values assumed by the Markov process in this instance reveals the full path, where the low values we seek do not always appear at the end of the iteration, but they are still identified. The global minimum, $\mathcal{F}_{\min}=-1$, was successfully found using our method. The trend line clearly illustrates the fluctuations, making them easier to observe. Similar to the previous experiment with the Ackley function, the high fluctuations in the generated surface explain the highly fluctuating path we observed.

\subsubsection{Differences with the Classical Simulated Annealing Algorithm} 

As in the previous case, we forget for a moment the imperative need for restricted domains. Therefore, only for diagnostic purposes will we compare our proposal with the classic SA. In this highly complex example, we found an acceptable mapping DNN for The Drop-Wave function \eqref{DropWaveFunction}. In Figure \ref{fig08}, we observed a similar behavior between this classic SA algorithm and our method until iteration $300$. However, the chaotic movement of the classic SA shifted significantly towards the opposite direction of the minimum value. Moreover, if we examine closely, these values do not make sense for the original equation \eqref{DropWaveFunction}, which is easily distinguishable in Figure \ref{fig05}.

\begin{figure}[H]
    \centering
    \includegraphics[width=0.6\linewidth]{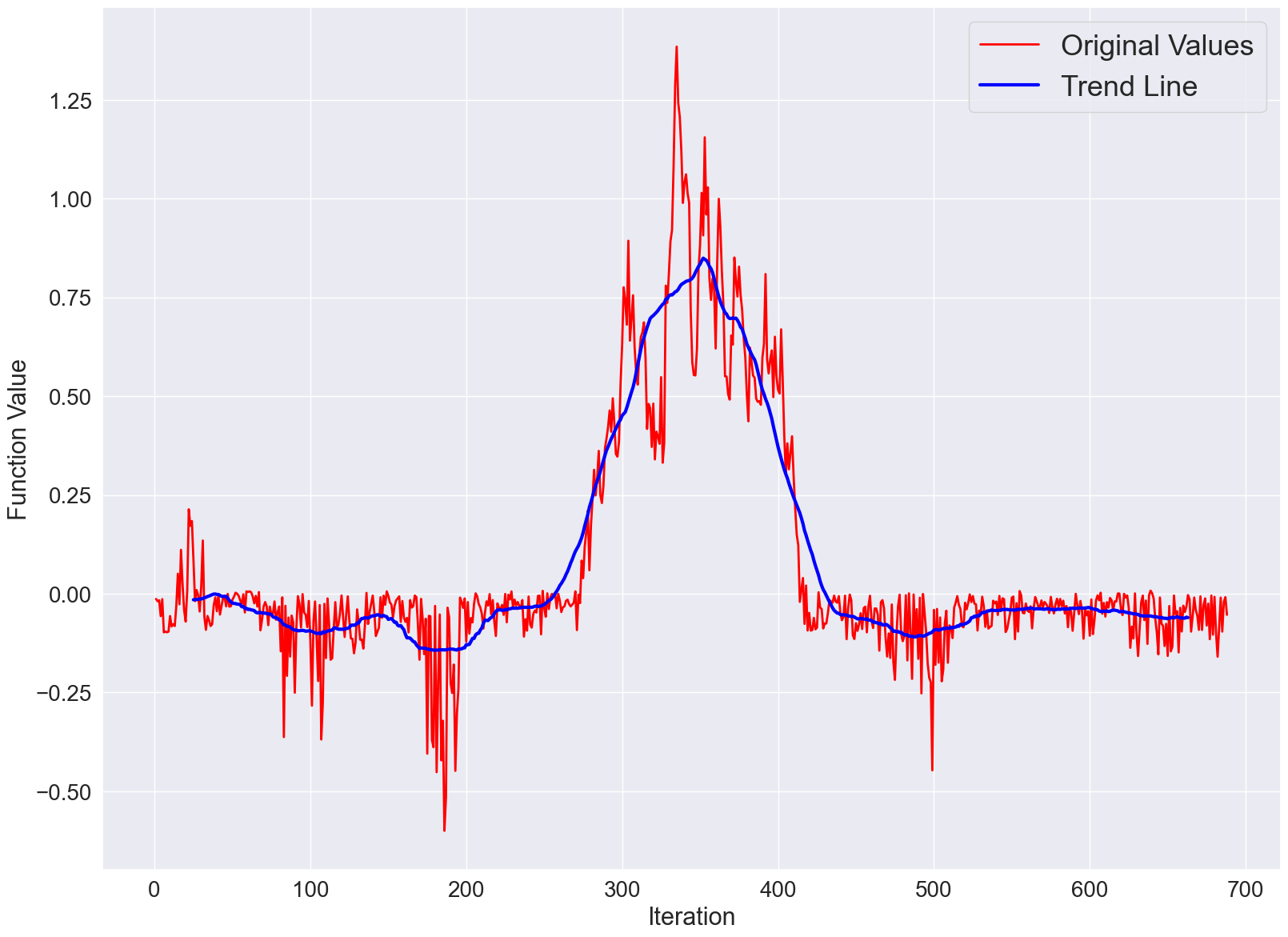}
    \caption{Red line: Sample path of the classical Simulated Annealing algorithm for the Drop-Wave function. Blue line: fitted trend line.}
    \label{fig08}
\end{figure}

\subsection{Multiple Global Minima Function}

Our method also can be applied in an space where there are multiple global minima. Also, it is not restricted for two dimension but more than two dimensions, and in particular n-dimensions, is allowed. For example, consider the function
\begin{align}\label{MultipleGlobalMinimaFunction}
   f(x, y, z) &= (x^2 - 1)^2 + (y^2 - 1)^2 + (z^2 - 1)^2,
\end{align}
where each term \((x^2 - 1)^2\), \((y^2 - 1)^2\), and \((z^2 - 1)^2\) achieves its minimum value of 0 when \(x^2 = 1\), \(y^2 = 1\), and \(z^2 = 1\), respectively. This results in two possible values for each variable: \(x = \pm 1\), \(y = \pm 1\), and \(z = \pm 1\). The combination of these values gives the following 8 global minima:
\begin{align*}
(x, y, z) \in \{&(-1, -1, -1),\, (-1, -1, 1),\, (-1, 1, -1),\\ &(-1, 1, 1),\, 
(1, -1, -1),\, (1, -1, 1),\, \\ &(1, 1, -1),\, (1, 1, 1)\}
\end{align*}

The function is evaluated on the square $E=[-3, 3]\times[-3,3]\times[-3,3]$. Then, we want to estimate the function \eqref{MultipleGlobalMinimaFunction} using only a discretized sample of points $\mathcal{D}=\{(x_{1i},x_{2i},,x_{3i},f_{i})\}_{i=1}^{m}\subset E$.

\begin{figure}[H]
    \centering
    \includegraphics[width=0.6\linewidth]{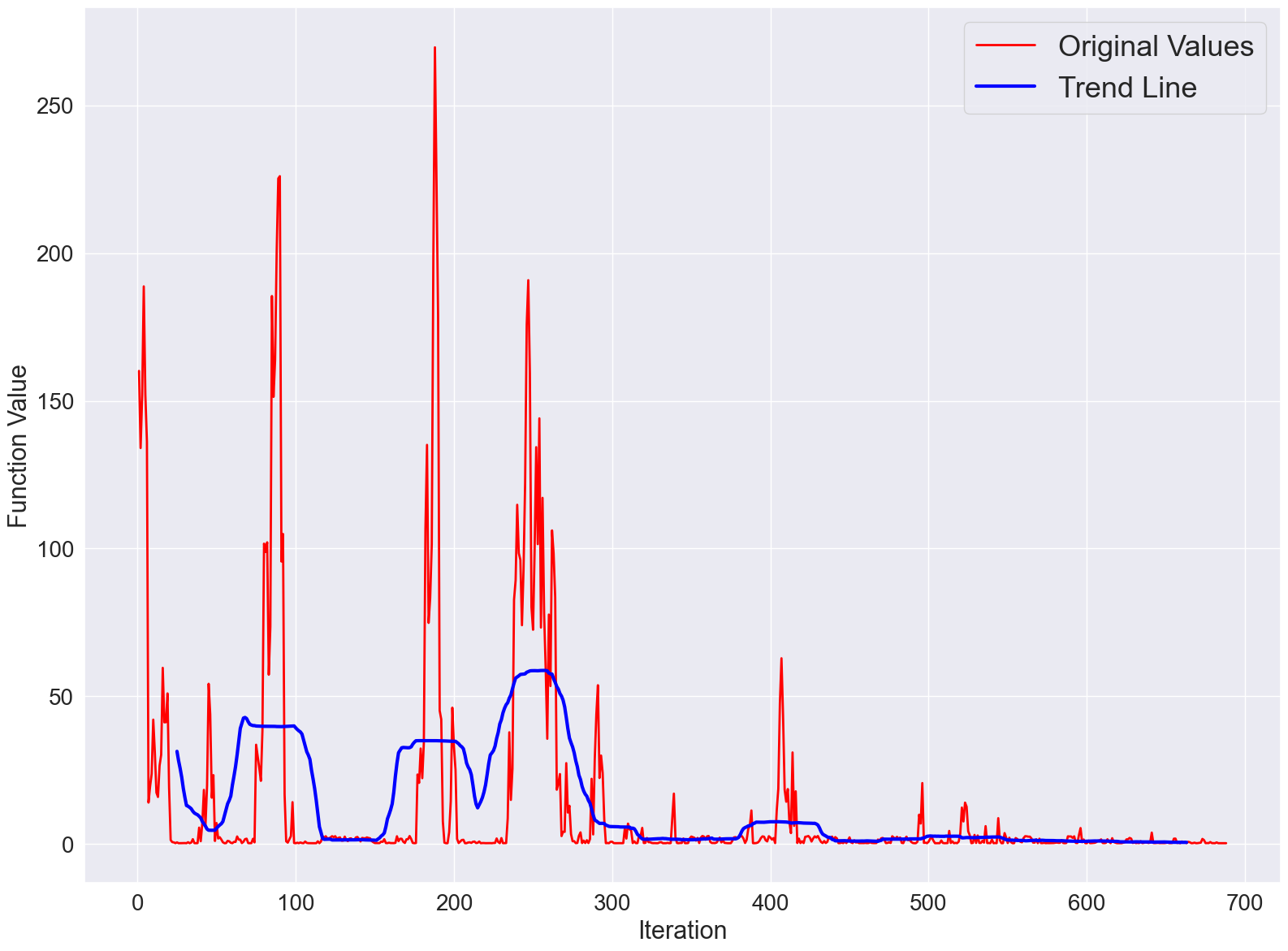}
    \caption{Red line: Sample path of Markov process $(X_i)_{i\in\mathbb{N}}$ for the Multiple Global Minima function. Blue line: fitted trend line.}
    \label{fig11}
\end{figure}

For this function, we use $L=7$ residual blocks, where $H_0 = 128$ represents the neurons in the input layer, $H_1 = H_2 = 256$, $H_l = 512$ for $l = 3, 4, 5$, $H_6 = 256$, and $H_7 = 128$, which represents the neurons in the final residual block. The output layer maps the final $128$ neurons to a single output. All layers in the network use the ReLU activation function, except for the final output layer. The model was configured with the following setup: The loss function utilized was Mean Squared Error (MSE). Training was performed using the Adam optimizer with a learning rate of $0.001$ and was performed over 1000 epochs. The fit generated by these ResNets is acceptable, with a mean absolute error of $1.4161$ and a mean squared error of $4.1011$, based on 1500 points generated within the range $[-3, 3]\times[-3,3]\times[-3,3]$. The Figure \ref{fig11} illustrates the sample path of our method, showing early convergence towards to one of the global minimum.

\subsubsection{Differences with the Classical Simulated Annealing Algorithm} 
{In this exampleinvolving a function with multiple global minima, we compare the performance of our method with the classical Simulated Annealing (SA) algorithm. The Figure \ref{fig11} illustrates the sample path of our method, showing early convergence towards to one of the global minimum and repetitive behavior in various iterations.} {In contrast, Figure \ref{fig12} depicts the sample path of the classical SA algorithm. While both methods exhibit comparable behavior in the early iterations, the classical SA algorithm demonstrates significant chaotic exploration and is slower in finding a global minimum. This behavior is evident after iteration $200$, where the function values deviate from convergence.}

\begin{figure}[H]
    \centering
    \includegraphics[width=0.6\linewidth]{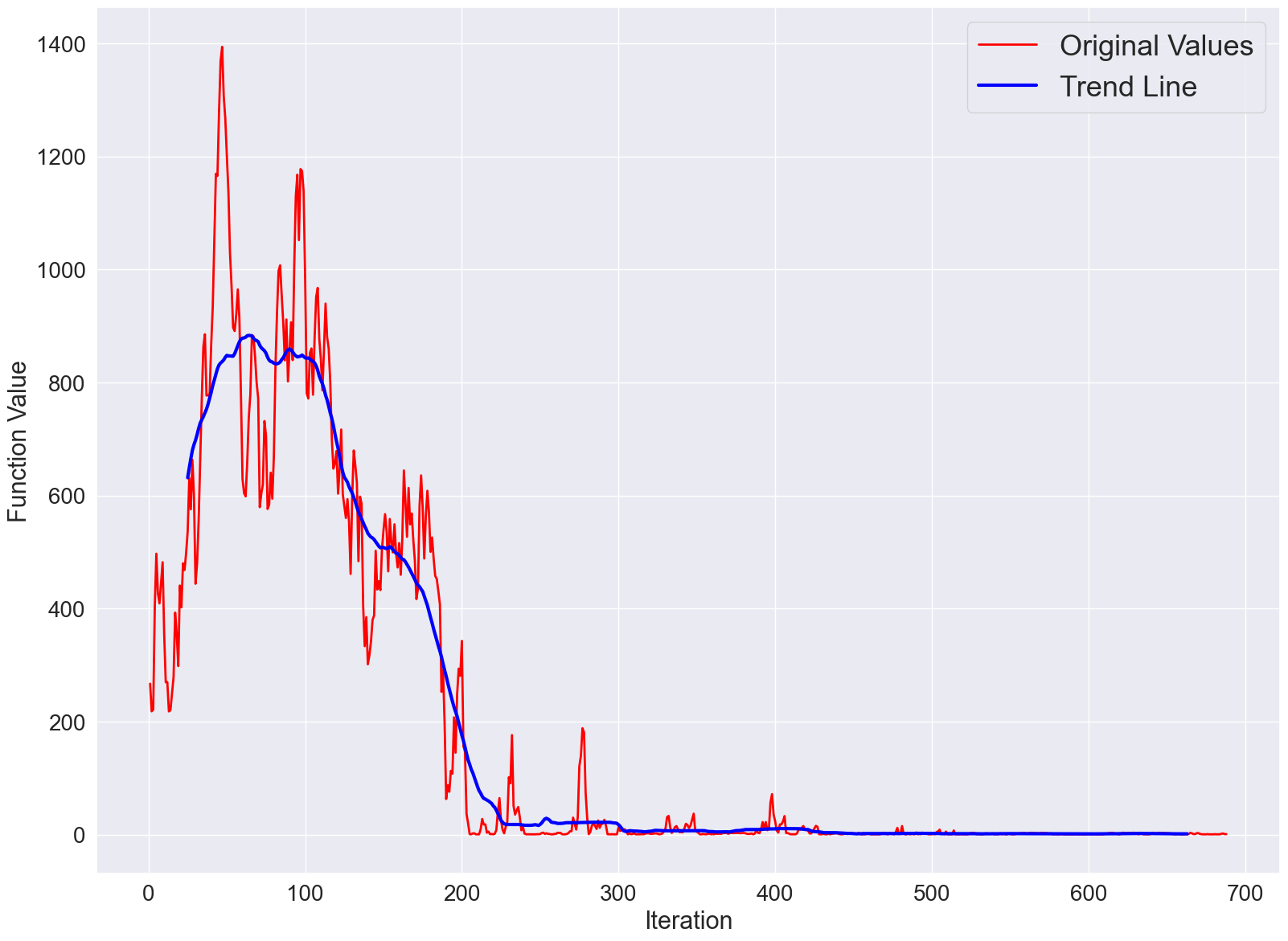}
    \caption{Red line: Sample path of the classical Simulated Annealing algorithm for Multiple Global Minima function. Blue line: fitted trend line.}
    \label{fig12}
\end{figure}

The Python codes of all the experimental evaluations are available in this \href{https://github.com/Nicerova7/output-range-analysis-for-deep-neural-networks-with-simulated-annealing}{link}.


\section{Conclusions}
    This work presents an algorithm based on Simulated Annealing (SA) for output range estimation in Deep Neural Networks (DNNs). This approach effectively adapts to restricted domains, overcoming the limitations posed by high non-linearity and the lack of local geometric information in DNNs. Through rigorous theoretical analysis, we demonstrate the algorithm’s convergence to optimal points, enabling accurate estimation of global minima and maxima.

 The experimental evaluation, using complex functions such as Ackley and Drop-Wave, validates the algorithm’s effectiveness, showcasing its ability to handle multiple local minima and multidimensional domains. Compared to classical Simulated Annealing, our method exhibits greater stability and faster convergence. With its simple implementation and versatility, this methodology is applicable to various DNN architectures, providing a practical and reliable tool for neural network verification and validation.

\section*{Data availability} 
The data and source code used in this study are publicly available at the following repository:  
\begin{center}
    \textbf{\href{https://github.com/Nicerova7/output-range-analysis-for-deep-neural-networks-with-simulated-annealing}{GitHub Repository: Output Range Analysis for Deep Neural Networks with Simulated Annealing}}
\end{center}

This study uses data and code to analyze the output range of deep neural networks based on simulated annealing processes. To ensure the reproducibility of the experiments and facilitate replication, the data and source code are freely accessible.

\subsection*{Source Data}  
If the data has been previously published, details of the dataset and where it can be accessed should be provided here.

\subsection*{Underlying Data}  
The repository contains all data used in this study, including input and output files, Python implementation code, and experimental results. The key contents include:

\begin{itemize}
    \item \textbf{README.md}: Provides an overview of the project, detailing the approach and methodology for output range analysis of deep neural networks using simulated annealing. It includes instructions for reproducing the results and relevant references.
    
    \item \textbf{neural\_network.py}: Implements the deep neural network architecture used in the study. It defines the \texttt{DeeperResidualNN} class, the loss function (\texttt{criterion}), the optimizer (\texttt{optimizer}), and the training loop for fitting the model to the dataset.
    
    \item \textbf{simulated\_annealing.py}: Implements the simulated annealing algorithm used to estimate the output range of the neural network. It defines the \texttt{simulated\_annealing} function, which optimizes over non-convex surfaces, including parameters such as initial solution, maximum and minimum temperature, cooling rate, number of iterations, intervals, and sigma.
\end{itemize}

The data and code are freely available without restrictions for use and reproduction.

\section*{Competing interests}
The authors declare that there are no financial, personal, or professional competing interests that could have influenced the content of this article.

\section*{Grant information}
The authors declare that no grants were involved in supporting this work.

\bibliographystyle{plainnat} 
\bibliography{references}

\end{document}